\documentclass{article}

\usepackage[nonatbib,preprint]{neurips_2024}

\usepackage{hyperref}
\hypersetup{
     colorlinks=true,
     linkcolor=red,
     filecolor=blue,
     citecolor = orange, 
     urlcolor=cyan,
}
\usepackage{url} 

\usepackage[numbers]{natbib}

\newcommand{\lv}{\left\vert}
\newcommand{\rv}{\right\vert}
\newcommand{\cF}{\mathcal{F}}
\newcommand{\cT}{\mathcal{T}}

\newcommand{\cM}{\mathcal{M}}
\newcommand{\cN}{\mathcal{N}}
\newcommand{\cZ}{\mathcal{Z}}
\newcommand{\cH}{\mathcal{H}}
\newcommand{\cA}{\mathcal{A}}
\newcommand{\bP}{\mathbb{P}}
\newcommand{\bR}{\mathbb{R}}
\newcommand{\bQ}{\mathbb{Q}}

\newcommand{\bW}{\mathbb{W}}
\newcommand{\bE}{\mathbb{E}}

\newcommand{\supp}{\operatorname{supp}}

\newcommand{\pr}{\operatorname{Pr}}
\newcommand{\tdT}{\Tilde{\boldsymbol{T}}}
\newcommand{\tv}{\operatorname{TV}}
\newcommand{\kl}{\operatorname{KL}}
\newcommand{\bdz}{\boldsymbol{z}}
\newcommand{\bdf}{\boldsymbol{f}}

\newcommand{\bdT}{\boldsymbol{T}}
\newcommand{\bdzeta}{\boldsymbol{\zeta}}

\usepackage{amsmath}
\usepackage{amssymb}
\usepackage{mathtools}
\usepackage{amsthm}
\usepackage{bbm}
\newtheorem{theorem}{Theorem}[section]
\newtheorem{proposition}[theorem]{Proposition}
\newtheorem{lemma}[theorem]{Lemma}
\newtheorem{corollary}[theorem]{Corollary}
\theoremstyle{definition}
\newtheorem{definition}[theorem]{Definition}
\newtheorem{assumption}[theorem]{Assumption}
\theoremstyle{remark}
\newtheorem{remark}[theorem]{Remark}
\newtheorem{example}[theorem]{Example}

\usepackage[utf8]{inputenc} 
\usepackage[T1]{fontenc}    
\usepackage{hyperref}       
\usepackage{url}            
\usepackage{booktabs}       
\usepackage{amsfonts}       
\usepackage{nicefrac}       
\usepackage{microtype}      
\usepackage{xcolor}         

\title{A General Theory for Compositional Generalization}

\author{%
  Jingwen Fu$^1$ , Zhizheng Zhang$^2$, 
  Yan Lu$^2$ , Nanning Zheng$^1$\thanks{Corresponding Authors}\\   \{fu1371252069@stu,nnzheng@mail\}.xjtu.edu.cn \\
  \{zhizzhang, yanlu\}@microsoft.com \\
  $^1$National Key Laboratory of Human-Machine Hybrid Augmented Intelligence,\\ National Engineering Research Center for Visual Information and Applications,\\ and Institute of Artificial Intelligence and Robotics, Xi'an Jiaotong University, $^2$Microsoft\\
}

\begin{document}

\maketitle

\begin{abstract}
Compositional Generalization (CG) embodies the ability to comprehend novel combinations of familiar concepts, representing a significant cognitive leap in human intellectual advancement. Despite its critical importance, the deep neural network (DNN) faces challenges in addressing the compositional generalization problem, prompting considerable research interest. However, existing theories often rely on task-specific assumptions, constraining the comprehensive understanding of CG. This study aims to explore compositional generalization from a task-agnostic perspective, offering a complementary viewpoint to task-specific analyses. The primary challenge is to define CG without overly restricting its scope, a feat achieved by identifying its fundamental characteristics and basing the definition on them. Using this definition, we seek to answer the question "what does the ultimate solution to CG look like?" through the following theoretical findings: 1) the first No Free Lunch theorem in CG, indicating the absence of general solutions; 2) a novel generalization bound applicable to any CG problem, specifying the conditions for an effective CG solution; and 3) the introduction of the generative effect to enhance understanding of CG problems and their solutions. This paper's significance lies in providing a general theory for CG problems, which, when combined with prior theorems under task-specific scenarios, can lead to a comprehensive understanding of CG.

\end{abstract}

\section{Introduction}
\label{sec:intro}
Compositional Generalization~(CG) represents the capacity to comprehend novel combinations of familiar concepts, an intellectual feat widely regarded as a pivotal milestone in human cognitive evolution \citep{pearl2018book,harari2014sapiens}. This remarkable ability empowers humans to generate an infinite array of ideas and constructs from finite building blocks of knowledge. However, machines have consistently struggled to emulate this level of compositional generalization, as it fundamentally challenges the prevalent assumption of independent and identically distributed (IID) data, a cornerstone principle in the machine learning literature~\citep{kawaguchi2017generalization,bartlett2002rademacher,bousquet2002stability,mohri2018foundations,mcallester1998some,fu2023generalization,fu2023learning}. When faced with the data significantly divergent from the training (support) distribution, achieving meaningful generalization becomes virtually insurmountable~\citep{koh2021wilds,sagawa2021extending,dong2022first}. This stark reality underscores the critical need for a rigorous theoretical examination of compositional generalization, as it holds the key to bridging the gap between human-like adaptability and the limitations of current machine learning models in handling unforeseen, complex combinations of concepts.

Despite numerous studies~\citep{wiedemer2024compositional,dong2022first,netanyahu2023learning,zhao2022toward} have tackled the issue of compositional generalization (CG), yet these inquiries have predominantly delved into isolated scenarios rather than offering a holistic perspective. This fragmented approach impedes a comprehensive understanding of CG. We assert that grasping the entirety of the CG problem is imperative for several reasons: 1) A comprehensive view is necessary to discern which aspects of findings are universally applicable to CG and which are contingent upon specific tasks. 2) Without a comprehensive understanding, it's challenging to contextualize specific problems within the broader CG task, hindering insight into its development and future research directions. 3) 
The significance of general analysis in advancing theoretical frameworks is confirmed by the works of other disciplines~\citep{keynes1937general,strehler1960general,marquis1996category}.
It's crucial to clarify that our stance does not advocate for the superiority of general theories over task-specific ones, nor do we propose replacing task-specific theories entirely. Rather, \textbf{we argue that a comprehensive understanding of CG necessitates integrating both general and task-specific theories, as they complement each other.}

The primary challenge in general analysis lies in defining the CG problem without constraining its scope. Existing definitions~\citep{dong2022first,wiedemer2024compositional,ren2024improving} tend to focus narrowly on specific tasks, because of the human tendency to approach problems through concrete examples. Moreover, methods employed in other fields to develop general theories or meta-analyses are unsuitable for addressing CG, as they may overlook its unique characteristics. To address this challenge, we reevaluate the fundamental properties intrinsic to CG, crafting a new definition that incorporates only these essential elements. These properties include:
1)  Each sample within a compositional generalization context originates from well-defined, predetermined concepts. 2) A clear compositional rule exists, enabling the systematic generation of novel, previously unseen samples from known ones based on this rule.

Given this definition, we arrive at the fundamental question regarding CG: \textit{what does the ultimate solution to CG look like?} The key findings are:

1) \textbf{No Free Lunch (NFL) Theorem}: We establish the presence of a No Free Lunch (NFL) theorem in the realm of compositional generalization. This theorem posits that there are no methodologies consistently outperforming others across the entire spectrum of CG problems. It underscores the absence of a universally superior approach, necessitating tailored methods for specific problem contexts. \textit{In essence, our general analysis of CG reveals the absence of general solutions.}

\textbf{2) Generalization Bounds}: 
To further understand the conditions for an effective CG solution, we present a bound applicable to any task-specific problem. Our findings suggest that an effective solution for CG should result in reduced mutual information between the learning algorithm's output and the composition rule, considering the distribution generating the training data. \textit{The generality of this bound lies in its applicability across various CG problems}, transcending limitations of prior bounds confined to specific types of CG problems.

\textbf{3) Generative Effect:} To enhance the understanding of CG problems and their solutions, we introduce the concept of the “generative effect” — the emergent effect when combining two concepts. The CG tasks are divided into two groups based on the presence of the generative effect to enhance our comprehension of solution design. A sufficient condition is given for addressing tasks without generative effects. Additionally, we posit that the generative effect will pose an unavoidable challenge for future endeavors.


\section{Motivation and Related Works}

Numerous prior studies delve into the theoretical analysis of compositional generalization. In our examination, we scrutinize these studies through the lens of the assumptions they adopt, which we categorize into two main groups: \textbf{method assumptions} and \textbf{task assumptions}. Method assumptions pertain to constraints placed on methodologies devoid of task-specific information, while task assumptions encapsulate any presumptions that elucidate task properties. Our primary focus lies on task assumptions, as they inherently circumscribe the breadth of the analysis. We regard most method assumptions as less critical for several reasons: 1) Firstly, prioritizing the analysis of methodologies exhibiting superior performance is paramount over those burdened by internal limitations. 2) Secondly, method assumptions typically aim to streamline analytical complexities rather than constrain the scope of our findings. 3) Thirdly, unlike tasks, method design is a product of deliberate human endeavor and thus more readily controllable.
\textbf{Due to these reasons, the essence of the “general theory” in this paper lies in the removal of task assumptions rather than method assumptions.}

We contend that previous research on analyzing compositional generalization has heavily relied on task assumptions throughout their processes. 1) Firstly, some studies~\citep{dong2022first,wiedemer2024compositional,ren2024improving} analyze compositional generalization under clear assumptions about the distribution or data generation process. These assumptions are inherently task-specific because they are tailored to particular tasks. 2) Secondly, another line of inquiry focuses on compositional generalization within specific domains such as natural language~\citep{petrache2024position,chomsky2002syntactic,partee1995lexical,gordon2019permutation}, reinforcement learning~\citep{silver2012compositional,li2021solving,sutton1999between,tasse2022skill,bacon2017option}, and object-centric image generation~\citep{wiedemer2023provable,brady2023provably}. Although these studies may not explicitly state task assumptions, their narrow focus indicates task specificity. Consequently, we classify all these works as task-specific analyses.

In this study, we adopt a distinct approach by initially analyzing compositional generalization within a general framework, gradually exploring how these assumptions influence the problem. Therefore, the primary divergence in our work lies in the minimal task assumptions at the outset, allowing us to disentangle conclusions derived from general properties of compositional tasks from those specific to compositional generalization tasks. Another significant departure is that our method takes into consideration the role of the learning algorithm. 
\textbf{In essence, we argue that our contributions offer a supplementary viewpoint on the CG issue in contrast to earlier efforts. We are confident that integrating our research with past analyses can enrich the broader understanding of CG.} More discussions about the other related works are given in Appendix \ref{app_sec:related_work}

\section{Problem Definition}

In this section, we provide the fundamental concepts for the CG problem (Section \ref{subsec:preliminary}), followed by the definition of the CG (Section \ref{subsec:CG_definition}).

\subsection{Preliminary}
\label{subsec:preliminary}

\textbf{Notations} \quad  In this paper, we employ $\bP$ to signify the distribution and $\bP(\cdot)$ to denote its corresponding density. $\Pr(\cdot)$ denotes probability. Bold symbols represent random variables, while unbold symbols represent their corresponding values. For a random variable $\boldsymbol{x}$, $\bP_{\boldsymbol{x}}$ represents its distribution. The calligraphic font is used to denote the space or learning algorithm.

\textbf{Learning algorithm} \quad We consider the learning problem on a data space $\cZ$ and a function space $\cF$, where $f:\mathcal{Z}\to \bR_{+} \in \cF$. Given a distribution $\bP_{\bdz}$ on this data space $\cZ$ and a function $f$, $err(\bP_{\bdz},f)=\bE_{\bdz \sim \bP_{\bdz}} f(\bdz)$ and given the samples $D_n\sim \bP_{\bdz}^{\otimes n}$, $err(D_n,f)=\frac{1}{n}\sum_{z\in D_n} f(z)$. The learning algorithm $\cA(\cdot)$ is the mapping from the data and the distribution on the function space. We consider the output of the learning algorithm as the distribution is aligned with the inner randomness of the learning process. We denote $\mathcal{A}(D_n)$ as the operation on dataset $D_n$, and we denote $\mathcal{A}(\bP_{\bdz})$ as the learning algorithm operates on the infinity data sampled from $\bP_{\bdz}$. The corresponding random variable and distribution are denoted as $\bdf$ and $\bP_{\bdf}$. 

\textbf{Subdistribution} \quad We denote the two compositional factors as $a\in A$ and $b\in B$. The other factors, including randomness, are denoted as $\zeta$.  We divided the whole data distribution $\bP_{\bdz}$ into several subdistributions based on the different values of the compositional factors. These distributions are $\lbrace \bP_{a,b} \rbrace_{a\in A, b\in B}$.
The $\bP_{a,b}(z)$ satisfies that $\bP_{a,b}(z)=\frac{\bP_{\bdz}(z)}{\bP(a,b)}\mathbf{1}_{(a_z=a) \land (b_z=b)}$, where $a_z$, $b_z$ are the corresponding fact value of the sample $z$.

\textbf{Distribution split} \quad We denote $E=A \times B$ as the all possible combinations of factor $a,b$. The set $E$ is further divided into support set $S$ and unknown set $U$, where $U \cap S=\empty$ and $U \cup S=E$. Similarly, the distribution is also divided into the support distribution $\bP_S=\lbrace \bP_{a,b} \rbrace_{(a,b)\in S}$ and $\bP_U=\lbrace \bP_{a,b} \rbrace_{(a,b)\in U}$. Similar, we denote $err(\bP_{S},f)=\bE_{(a,b)\in S}[ err(\bP_{a,b},f)]$ and similar for $err(\bP_U,f)$. We denote $\bdf_{S}\sim \mathcal{A}(\bP_S)$.

\subsection{Compositional Generalization}
 \label{subsec:CG_definition}

Defining a problem is often the most challenging aspect. Previous definitions of compositional generalization have been restricted to specific contexts, thereby limiting their applicability. For example, some researchers~\citep{chomsky2002syntactic,lake2019human,montague1970universal} describe compositional generalization in language as “The algebraic capacity to understand and produce an infinite number of utterances from known components”. However, this definition is unclear about its applicability beyond language and lacks a formal framework. \citet{wiedemer2024compositional} provide a formal definition, but it is confined to object-centric image understanding. The definition by \citet{zhao2022toward} is specific to reinforcement learning scenarios. Although \citet{dong2022first} offer a definition that is not restricted to a particular application area, they impose stringent constraints on distributions.

The key challenge in developing a general definition of compositional generalization lies in identifying its intrinsic nature, which represents the inseparable core of CG. In this paper, we argue that the compositional generalization problem should possess the following properties:

(1) Each individual sample within a compositional generalization context is fundamentally derived from well-defined, predetermined concepts. 

(2) There is a clear compositional rule, and the novel, previously unseen samples can be systematically generated from existing, known samples based on the rule.  

Condition (1) ensures the existence of basic units, i.e., concepts, that make composition possible. Condition (2) ensures a connection between samples with different concepts. Both conditions are essential to the CG problem: without condition (1), “composition” cannot be explained, and without condition (2), “generalization” cannot be addressed. Building on this intuition, we provide the following formal definition of compositional generalization, which incorporates both conditions. The definition is given below:
\begin{definition}
    \label{def:cg}
    The distributions $\lbrace \bP_{a,b}, (a,b) \in E \rbrace$ are compositional distributions if they satisfy:
    \begin{itemize}
    \item (1) (Well-defined concepts) For $a_1,b_1,a_2,b_2$, if $a_1 \neq a_2$ or $b_1 \neq b_2$, we have $ \supp \bP_{a_1,b_1} \cap \supp \bP_{a_2,b_2}=\emptyset$.
    \item (2) (Compositional rule) For all $a_1,a_2,b_1,b_2$, there exsits a measurable bijection function $T_{a_1 \to a_2,b_1 \to b_2}$, such that for all $z \subset \supp{\bP_{a_1,b_1}}$ we have $\bP_{a_1,b_1}(z)=\bP_{a_2,b_2}(T_{a_1 \to a_2,b_1 \to b_2}(z))$ and $T_{a_1 \to a_1,b_1 \to b_1}(z)=z$.
    \end{itemize}
\end{definition}
\begin{remark}
    \textbf{Definition (1)} ensures that each example is derived from well-defined, predetermined concepts. If this condition is violated, there exist $a_1,b_1$, $a_2,b_2$, and $z$ such that $a_1\neq a_2$ or $b_1\neq b_2$, yet $\bP_{a_1,b_1}(z)>0$ and $\bP_{a_2,b_2}(z)>0$. This implies that concept value $z$ satisfies both $a_z=a_1$, $a_z=a_2$, $b_z=b_1$, and $b_z=b_2$. Since $a_z$ and $b_z$ can only take one value, at least one equation is violated. \textbf{Definition (2)} ensures that for all $(a_1,b_1),(a_2,b_2)$, we can derive $\bP_{a_2,b_2}$ given $\bP_{a_1,b_1}$ and $T_{a_1 \rightarrow a_2, b_1 \rightarrow b_2}$, and vice versa. The function $T$ serves as the compositional rule for the CG problem.
\end{remark}

\begin{remark}
    In this paper, we primarily focus on compositional generalization involving two concepts. This definition and our conclusions can be readily extended to CG scenarios with more than two concepts. Analyzing CG with two concepts serves as a fundamental basis for addressing more complex CG problems~\citep{dong2022first,wiedemer2024compositional,ren2024improving,petrache2024position,chomsky2002syntactic,partee1995lexical,gordon2019permutation,silver2012compositional,li2021solving,sutton1999between,tasse2022skill,bacon2017option,wiedemer2023provable,brady2023provably}.    
\end{remark}

\begin{example}
    (Image) Within the realm of single-object images, let set $A$ signify the object's shape, and set $B$ denote its size. We define $\bP_{a,b}$ as the distribution of images with shape $a$ and size $b$. One possible mapping $T$ could be a function solely altering the object's shape and size while maintaining other attributes like color and position.
\end{example}
\begin{example}
    (Robot) We examine a task distribution for robots comprising one walking task and one operational task. Let set $A$ denote a sequence of walking sub-tasks, while set $B$ signifies a collection of operational sub-tasks. We define distributions such as $\bP_{a_1,b_1}$ to cover tasks involving slow walking and object picking, $\bP_{a_2,b_1}$ for tasks involving regular walking and object picking, and $\bP_{a_1,b_2}$ for tasks involving slow walking and object stacking. The target distribution, labeled as $\bP_{a_1,b_2}$, is tailored for tasks specifically involving slow walking and object stacking.
\end{example}

\begin{remark}
    In certain compositional generalization tasks, neural networks are required to first master basic concepts before progressing to more complex challenges. Consider robot learning as an example: initial tasks may concentrate solely on activities like walking and fetching objects independently. Subsequently, the network must extend its comprehension to situations where the robot performs both activities simultaneously. In such instances, we introduce the null factor $\emptyset$. The distribution pertaining to individual concept can be represented as $\mathbb{P}_{a,\emptyset}$ or $\mathbb{P}_{\emptyset,b}$. We can adjust $A'=A\cup { \emptyset}$ and $B'=B\cup { \emptyset}$.
\end{remark}


\begin{definition}
    \label{def:solvability}
    (Solvable) A CG problem is solvable if there exsits a learning algorithm $\mathcal{A}$, such that $err(\bP_U,\bdf_{S})=\mathcal{O}(\epsilon)$, where $\epsilon=\max \limits_{(a,b)\in E} \min \limits_{f \in \mathcal{F}} err(\bP_{a,b},f)$.
\end{definition}

\begin{remark}
     Recall that $\bdf_S\sim\mathcal{A}(\bP_S)$, meaning $\bdf_S$ is trained on an infinite dataset sampled from $\bP_S$. The solvability of a CG problem implies the existence of a learning algorithm capable of achieving a low error when trained on infinitely large samples from the support distribution. 
\end{remark}

\section{No free lunch theorem}

When approaching a problem, it is natural to ask what the solution will look like. The key question is whether a general solution for a set of tasks is feasible or if task-specific solutions are necessary. To answer this, we must examine the relationship between the solutions for different tasks. If the solution for one type of task is also effective for another type, then a general solution may be possible. However, if a solution for one task performs poorly in other tasks, it is crucial to develop task-specific solutions rather than a general one.

To begin our analysis, we define the surrogate function, a methodological technique used for analysis:
\begin{definition}
    (Surrogate function) Given a base distribution $\bP_{a_0,b_0}$ where $(a_0,b_0) \in S$, We say $(\Tilde{f},\Tilde{T})$ is a surrogate function of function $f$ if for all $(a_1,b_1) \in S$, we have $err(\bP_{a_0,b_0},f)=err(\bP_{a_0,b_0},\Tilde{f})$ and $err(\bP_{a_1,b_1},f)=err((T_{a_0 \to a_1,b_0\to b_1})_* \bP_{a_0,b_0},\Tilde{f})$, where $(\cdot)_*$ denotes the pushforward operation.
\end{definition}

Then, we present the assumptions for the analysis below:
\begin{assumption}
    \label{ass:simplify_func_learning_algorithm}
    The function space $\mathcal{F}$ and the learning algorithm $\mathcal{A}(\cdot)$ satisfies the following properties:
    \begin{itemize}
        \item ($i$) (Surrogate) There exsits a base distribution $\bP_{a_0,b_0}$  and a surrogate generation function $\mathcal{B}(\cdot)$, where $(\Tilde{f},\Tilde{T}_f)=\mathcal{B}(f)$,  such that for all $f_1,f_2\in \supp \bP_{\bdf_S}$, we have $(\Tilde{T}_{f_1}=\Tilde{T}_{f_2}) \Longleftrightarrow (f_1=f_2)$.
        \item ($ii$) (Consitency) For all valid distributions $\lbrace \bP_{a,b}^{(T)} \rbrace_{(a,b)\in E}$, we have $\epsilon=\max \limits_{(a,b)\in E} \min \limits_{f \in \mathcal{F}} err(\bP_{a,b},f)=0$ and there exists $f\in \mathcal{F}$, such that $\Tilde{T}_f=T$.
        \item ($iii$) (Convergence) $\max \limits_{(a,b)\in S} err(\bP_{a,b},\bdf_S)=0$, where $\bdf_S \sim \mathcal{A}(\bP_S)$.
    \end{itemize}
\end{assumption}

\begin{remark}
    The Assumption ($i$) is to reduce the difficulty of the analysis. Similar results can also be infered by removing Assumption ($i$) but this can only introduce more difficulty in analyzing instead of giving more insights. The Assumption ($ii$) and ($iii$) is to make sure both the function space $\mathcal{F}$ and learning algorithm $\mathcal{A}(\cdot)$ are will designed. Therefore, we regard the ($i$) and ($iii$) as the method assumption as they will not constrain the applicable tasks. The assumption $(ii)$ contains certain constraint about the tasks that is it ensure there exists $f\in \mathcal{F}$, such that $\Tilde{T}_f=T$. We believe that this constraint is quite weak because 1) the function space $\mathcal{F}$ is usually to be quite large in practice which makes the constraint trival and 2) many theoretic analysis target for general problems incorporate similar assumption and our work is in line with them. As a result, this assumption will not impact the generality of the analysis.
\end{remark}

\begin{definition}
    Given two composition rule $T,T'$, if $T=T'$, then for all $(a_1,b_1), (a_2,b_2) \in S$, we have $T_{a_1 \to a_2, b_1\to b_2}=T'_{a_1 \to a_2, b_1\to b_2}$.
\end{definition}

Based on the Assumption \ref{ass:simplify_func_learning_algorithm}, we have the following conclusion:
\begin{theorem}
\label{thm:NLF}
Under Assumption \ref{ass:simplify_func_learning_algorithm}, for all valid division of $S$ and $U$, and any $(\mathcal{A}_1,\mathcal{F}_1),(\mathcal{A}_2,\mathcal{F}_2)$, satisfying $|\mathcal{F}_1|=|\mathcal{F}_2|$, we have
    \begin{equation}
    \label{eq:NFL}
    \sum_{T\in \cT} \pr(\Tilde{\bdT}=T|\mathcal{F}_1,\mathcal{A}_1,\bP_{S}^{(T)})=\sum_{T\in \cT} \pr(\Tilde{\bdT}=T|\mathcal{F}_2,\mathcal{A}_2,\bP_{S}^{(T)}),
\end{equation}
    where $\Tilde{\bdT}=\Tilde{T}_{\bdf_S}$ and $\bP_{S}^{(T)}$ is the support distribution generated using the compositional rule $T$.
\end{theorem}

\textbf{Proof Sketch:} We use the methematical induction to prove the theorem:

\textbf{Step 1:} we consider $S=\emptyset$, then it is obvious that the Equation \ref{eq:NFL} holds,

\textbf{Step 2:}  We consider two set $S_1\subset S_2 \subset S$ and $|S_2|=|S_1|+1$. Assume that $\sum_{T\in \cT} \pr(\Tilde{\bdT}=T|\mathcal{F}_1,\mathcal{A}_1,\bP_{S_1}^{(T)})=\sum_{T\in \cT} \pr(\Tilde{\bdT}=T|\mathcal{F}_2,\mathcal{A}_2,\bP_{S_1}^{(T)})$ holds, then we can prove that $\sum_{T\in \cT} \pr(\Tilde{\bdT}=T|\mathcal{F}_1,\mathcal{A}_1,\bP_{S_2}^{(T)})=\sum_{T\in \cT} \pr(\Tilde{\bdT}=T|\mathcal{F}_2,\mathcal{A}_2,\bP_{S_2}^{(T)})$ using the properties in Definition \ref{def:cg}.

\textbf{Step 3:} By iterative applying the \textbf{step 2}, we can obtain Equation \ref{eq:NFL}.

\textbf{Intepretation}\quad The sum over all $T\in \mathcal{T}$ can be regarded as add the performance over all possible compositional generalization problems. Each $T$ can be regarded as one problem.  The result indicates that the performance is not better than randomly choose a function from the function space when consider the peformence from all the compositional task. Therefore, it is possible to improve the performance on the task generated by a specific function $T$, but it is impossible to design one method $(\mathcal{A}, \mathcal{F})$ that achieve better than other method over all the compositional generalization problem.  The NFL theory emphasizes the impossibility for a universally solution. Therefore, a more fruitful direction involves exploring strategies for leveraging task-specific information. Such information is provided by making task assumptions and leverage these assumptions to design specific methods.

\textbf{Contribution.} \quad The key contribution of the Theorem \ref{thm:NLF} is to consider the performance of the methods (i.e., $(\mathcal{A},\cF)$) across all possible tasks with different composition rule $T\in \cT$. From this perspective, we can relaxed our attention from a specific tasks to the broader situations. And the result indicates that a general method to solve all CG problems cannot be found.

\textbf{Compared with prior works.} \quad \textbf{1) Compositional Generalization.} There are some previous works~\citep{dong2022first,dziri2024faith} discussing about the impossibility for us to solve specific compositional generalization problems without under certain conditions. Our work advances previous understanding of the compositional generalization in that instead of finding specific conditions or situations where the CG problem is unsolvable, our method tackle a general situation. \textbf{2) No free lunch Theorem.} The “No free lunch” Theorem are first proposed in optimization~\citep{wolpert1997no} and the search problems~\citep{wolpert1995no}. The others extend this theorem in the learning problem, mainly focused on the supervised learning~\citep{sterkenburg2021no,wolpert2021important,wolpert2002supervised}. More details about the no free lunch theorem can be found in the references~\citep{adam2019no,joyce2018review,ho2001simple,ho2003no,yang2012free,rowe2009reinterpreting}. In this paper, we extend the “No Free Lunch” Theorem in Compositional Generalization. The key different of our paper compared with prior works in that the mainly focus of our paper lies in the compositional rule $T$ and our theorem isn't limited any perticular learning problem, e.g. supervised learning and unsupervised learning, as long as the tasks are CG in sense of Definition \ref{def:cg}.

\section{Generalizaton Bounds}

The analysis in the previous section indicates that there are no universal solutions for CG problems. In this section, we aim to explore the properties of effective solutions for CG by devising a new generalization bound. By minimizing this bound, we can infer these properties. To start, we will present the assumption and definition.


\begin{assumption}
    \label{ass:GB_convergence}
    The learning algorithm satisfies that for all $S\subset E$, we have $err(\bP_S,\bdf_S)=\mathcal{O}(\epsilon)$, where $\epsilon=\max \limits_{(a,b)\in E} \min \limits_{f \in \mathcal{F}} err(\bP_{a,b},f)$.
\end{assumption}
\begin{remark}
    The condition is to ensure that the learning algorithm can learning algorithm can find a good convergence point for possible any distributions. 
\end{remark}
\begin{definition}
    ($L$-bounded) We say a function $h(\cdot)$ is $L$, if for all valid input $x$, we have $|h(x)| \leq L$.
\end{definition}

\begin{theorem}
\label{thm:information}
Given training data $D_n\in \mathcal{Z}^n$ sampled from the support distribution $\bP_S$, learning algorithm $\mathcal{A}$ and function space $\mathcal{F}$, if $err$ is $L$-bounded, then we have
\begin{equation*}
    \lv \bE_{D_n,f\sim \mathcal{A}(D_n)} [err(\bP_U,f)-err(D_n,f)] \rv  \leq GenIID+ \kappa_n L\Phi(I(\bdf_S;\bdT|\bP_{S}^{(\bdT)}))+\mathcal{O}\left( \epsilon \right),
\end{equation*}
where $\epsilon=\max \limits_{(a,b)\in E} \min \limits_{f \in \mathcal{F}} err(\bP_{a,b},f)$, $I(\cdot;\cdot)$ denotes the mutual information, $GenIID$ denote any generalization error bound with IID assumption, $\Phi(x)\triangleq \sqrt{\min\lbrace x/2,1-\exp(-x) \rbrace}$ and $\kappa_n \triangleq \frac{\lv \bE_{D_n \sim \bP_{S}}[err(\bP_U,\bdf_{D_n})-err(\bP_S,\bdf_{D_n})]\rv}{\lv err(\bP_U,\bdf_{S})-err(\bP_S,\bdf_{S})\rv}$ (note that $\bdf_{D_n} \sim \mathcal{A}(D_n)$).
\end{theorem}

\begin{remark}
$\kappa$ quantifies the variation in performance gap between the support distribution and the target distribution across varying numbers of training samples.
 The $\kappa$ satisfies that $\lim_{n \to \infty} \kappa_n =1$. Recall that $\lim_{n\to \infty}\bdf_{D_n}=\lim_{n\to \infty}\mathcal{A}(D_n)=\mathcal{A}(\bP_S)=\bdf_S$. Then, we have $ \lim_{n \to \infty} |\bE_{D_n \sim \bP_{S}}[err(\bP_U,\bdf_{D_n})-err(\bP_S,\bdf_{D_n})]|= |err(\bP_U,\bdf_{S})-err(\bP_S,\bdf_{S})|$. 
\end{remark}

\begin{remark}
     The $GenIID$ term is one of a upper bound for $err(\bP_{S},f)-err(D_n,f)$. We can use any method to obtain the bounds, including uniform stability~\citep{hardt2016train}, information-thoretic methods~\citep{russo2016controlling,fu2023generalization} and Rademacher Complexity. If we use the Rademacher complexity, we have $GenIID=2 \mathcal{R}_n(\cF)$, where $\mathcal{R}_n=\bE_{\sigma,D_n}[\sup_{f\in \cF} \frac{1}{n}\sum_{i=1}^n \sigma_i f(x_i)]$ with $\sigma_i$ being independent uniform random variables taking values in $\lbrace -1, +1 \rbrace$.
\end{remark}
\begin{remark}
    The key term of the Theorem \ref{thm:information} is $\Phi(I(\bdf;\bdT|\bP_{S}^{(\bdT)}))$. This terms establishes a connection between compositional generalization and the information theory. It suggests that achieving a small CG error is probable if $\bdf_S$ exhibits lesser dependence with $\bdT$ given $\bP_{S}^{(\bdT)}$.The independence of $\bdf_S$ from $\bdT$ implies that the performance of $\bdf$ cannot be improved even with additional information about $\bdT$ given the support distribution $\bP_{S}^{(\bdT)}$.
\end{remark}
\textbf{Technique contribution} \quad We regard the compositional generalization as a specific kind of the out of distribution generalization, in line with the recent works~\citep{netanyahu2023learning,netanyahu2023learning,qiu2021improving,oren2020improving,hosseini2022compositional}. Therefore, the performance of the compositional generalization can be divided into two parts: the first part is the in-distribution learning and the second part is the influence of distribution shift. By considering the surrogate function, we can decompose these two parts and take our attention on the second part, which is the core of CG problems.

\paragraph{Tightness} To illustrate the tightness of our bounds, we compare our bound with that of \citet{ben2010theory}, which is a general bound for out of distribution generalization, and therefore, it is comparable to us. The detail is given in Appendix \ref{app_sec:tightness}. Here, we list the results of the comparison. We find that we can not simply say that one method is tighter to the other. We divided the CG tasks into two kind: the learning algorithm dominiated one and the function dominated one. \textbf{In the first situation}, the performance of the CG is highly depended on the learning algorithm and our bound is much better than \citet{ben2010theory}. This is reasonable because \citet{ben2010theory} doesn't consider the influence of the learning algorithm. \textbf{When comes to the second situation}, we find that our bound is better when given relative good support distributions, i.e. $|S|$ is large. On the other hand, the \cite{ben2010theory} is better.


\paragraph{Proof Sketch} The full version of the proof is given in Appendix \ref{app_sec:learnablity}. Here the sketch of the proof is given, which is divided into three steps: 

\textbf{Step 1:} We decomposite the generalization behavior into two terms:
\begin{equation}
     err(\bP_{U},f)-err(D_n,f)  = \underbrace{ err(\bP_{S},f)-err(D_n,f) }_{\text{IID~gen~error}} + \underbrace{ err(\bP_{U},f)-err(\bP_{S},f) }_{\text{CG~gen~error}} 
\end{equation}

\textbf{Step 2:} We replace ``IID gen error" term with $GenIID$ to indicate any generalization bound with IID assumption. The motivation is that the generalization behavior under IID situation is out of the scope of this paper. 

\textbf{Step 3:} For the ``CG gen error" term, we first use the $\kappa_n$ to decouple the influence of insufficient data. Then, with the Assumption \ref{ass:GB_convergence}, we can transform the term into $err(\bP_{U},\bdf_S)- err(\bP_{U},\bdf_E)+\mathcal{O}(\epsilon)$. By leveraging the knowledge of Wasserstein Distance and its connection to mutual information, we obtain the upper bound $\kappa_n L\Phi(I(\bdf_S;\bdT|\bP_{S}^{(\bdT)}))+\mathcal{O}(\epsilon)$.

\textbf{Step 3:} Combining the results of ``Step 1" and ``Step 2", we obtain the proposition. 

\begin{corollary}
\label{cor:bound}
    Given the support distribution $\bP_S$, if $err$ is $L$-bounded, and $\lim_{n \to \infty} GenIID =0$,
    the CG problem is solvable if $err(\bP_S,\bdf_S)+\Phi(I(\bdf_S;\bdT|\bP_{S}^{(\bdT)}))=\mathcal{O}(\epsilon)$.
\end{corollary}

\begin{remark}
    According to the definition \ref{def:solvability}, we consider the situation where $n \to \infty$ to check when the CG problem is solvable or not. The $GenIID\to 0$ when $n\to \infty$. Therefore, we have  $\lim_{n \to \infty}\lv \bE_{D_n,f\sim \mathcal{A}(D_n)} [err(\bP_U,f)-err(D_n,f)] \rv=\lv [err(\bP_U,\bdf_S)-err(\bP_S,\bdf)] \rv$.
\end{remark}

\begin{remark}
    Remind that $GenIID$ can be equal to any generalization bound under IID condition. There are many IID generalization bound that can ensure $\lim_{n \to \infty} GenIID\to 0$, including the uniform convergence methods, algorithm stability methods and information theoretic methods.
\end{remark}

\begin{remark}
    This generalization bound is special for certain learning algorithm $\mathcal{A}$ and function space $\mathcal{F}$. If we can select the learning algorithm and the function space that take into consider of the specific task, then we can achieve small $\Phi(I(\bdf_S;\bdT|\bP_{S}^{(\bdT)}))$. Usually, the information of the specific task is given by the task assumption. If we make enough task assumption, by taking the assumptions into the design of the learning algorithm and function, we can achieve obtain a small $\Phi(I(\bdf_S;\bdT|\bP_{S}^{(\bdT)}))$.
\end{remark}

\textbf{Compared with Prior Works} \quad Previous works \citep{netanyahu2023learning, dong2022first} provide generalization bound solutions for CG problems that are tailored to specific tasks. In contrast, our method offers the following unique characteristics: 1) Our bound is universally applicable to all problems, without restriction to any specific problem. 2) Our bound connects the generalization behavior with the mutual information “$I(\bdf_S;\bdT|\bP_{S}^{(\bdT)})$”; this connection allows for a deeper understanding of the conditions necessary for effective solutions to CG problems.


\section{Generative Effect}
\label{sec:Generative_Effect}

The previous sections established the proof that there are no general solutions for CG and outlined the properties of effective CG methods by deriving a new generalization bound. In this section, we advance our analysis of CG solutions by examining the generative effect. We start by considering the problem without the generative effect, which we term the Independent Rule Mechanism (IRM). Then, we proceed to discuss the problem with the generative effect.


\subsection{Independent rule mechanism}

In many cases, we can expect that the influence of one factor is independent of another. For example, we expect the position of one object to be independent of the position of another object. This independence is not in the statistical sense. We refer to it as independence because the effect of the first factor is unrelated to the effect of the second factor. The formal definition is given below:

\begin{definition}
    \label{def:IRM}
    We say a CG problem has IRM if the compositional rule is decomposition, i.e.  for all $(a_1,b_1),(a_2,b_2)\in E$, we have  $T_{a_1 \to a_2,b_1 \to b_2}=T_{a_1 \to a_2}\circ T_{b_1 \to b_2}=T_{b_1 \to b_2}\circ T_{a_1 \to a_2}$.
\end{definition}

\begin{remark}
The definition of IRM can be understood as follows: if we change one part of the whole, it will not affect another part. We argue that many prior CG works adhere to the IRM principle. Here are some examples: 1) In object-centric tasks~\citep{wiedemer2023provable,brady2023provably}, $a$ and $b$ can represent information about different objects. If we change one object, the others remain unchanged. 2) In reinforcement learning problems~~\citep{silver2012compositional,li2021solving,sutton1999between,tasse2022skill,bacon2017option} with two independent subtasks, changing one task does not affect the other task.
\end{remark}

\begin{assumption}
    \label{ass:IRM}
    The compositional generalization problem satisfies the following properties:
    \begin{itemize}
        \item ($i$) (C-Support) The CG problem has C-support (C is short for compositional), i.e., $S$ satisffies that 1) for all $a\in A$, there exsits $b\in B$ such that $(a,b) \in S$ and 2) for all $b\in B$, there exsits $a\in A$ such that $(a,b) \in S$.
        \item ($ii$) (Identifiable) The compositional rule is indentifiable, i.e. for all $(a_1,b_1),(a_2,b_2)\in S$ we can recover $(T_{a_1 \to a_2},T_{b_1 \to b_2})$ given the two distribution $ \bP_{a_1,b_1}$ and $ \bP_{a_2,b_2}$.
    \end{itemize}
\end{assumption}
\begin{remark}
    The assumption is stronger than the Assumption \ref{ass:simplify_func_learning_algorithm}. The key difference is that this assumption contrains the distribution and true compositional rule, which are task specific information, whether Assumption \ref{ass:simplify_func_learning_algorithm} is mainly about the method, which will not limit the scope to specific task. 
\end{remark}

\begin{assumption}
    \label{ass:uni_approx}
    The function space $\mathcal{F}$ has unversal approximation ability, i.e. for any valid function $h$, we have $f\in \cF$ such that $f=h$.
\end{assumption}

\begin{remark}
    This assumption is aligned with the research works that point out the neural networks are universal approximator. Current mainstream used neural network has been approved to have universal approximation ability.
\end{remark}
\begin{theorem}
    \label{thm:IRM}
    The CG problem with IRM is solvable under the Assumption \ref{ass:IRM} and Assumption \ref{ass:uni_approx}.  
\end{theorem}

\begin{remark}
    This Theorem states that we can obtain the solution for the CG problem with IRM by fully leveraging the information in Assumption \ref{ass:IRM}.
\end{remark}

\textbf{Prior works} \quad There are a lot of works~\citep{netanyahu2023learning,wiedemer2024compositional} that aim to find a provable solution for compositional generalization. We want to argue that the the solutions of most prior works are aligned with the independent relation mechanism. The detail analysis is given in Appendix \ref{app_sec:gen_effect}.

\subsection{Generative effect}

The IRM applies to problems where the whole is exactly the “sum” of its parts. Typically, however, the whole is “greater” than its parts. Drawing on the works of \citet{adam2019generativity, adam2019mathematical, adam2019abstract} in category theory, we describe this as the generative effect. Below, we provide the formal definition of the generative effect.

\begin{definition}
    (Generative effects) A CG problem has generative effects if it doesn't have the IRM. 
\end{definition}
\begin{remark}
    The generative effect implies that different factors are interdependent in the CG problem. If one factor changes, the effect of the other factors will also change. For CG problems with generative effects, we need to study not only the mechanism of each individual concept but also understand how different concepts interact and cooperate with each other. 
\end{remark}

In the following, we give the simple examples of the generative effect.

\begin{example}
    \label{ex:gen_effect}
    The factors $a,b$ are integers that take value in $(0,10]$. The $z$ is also take interger value in $(0,100]$. The distributon is assigned as $\bP_{a,b}(a \times b)=1$ and $\bP_{a,b}(z)=1$ for $z\neq a\times b$. Under this setting, we have $T_{a_1 \to a_1, b_1\to b_2}(z)=z+a_1\times(b_2-b_1)$. We cannot have the decomposition that $T_{a_1 \to a_1, b_1\to b_2}=T_{a_1 \to a_1} \circ T_{b_1\to b_2}=T_{b_1\to b_2}$, because $T_{b_1\to b_2}$ has to depend on the value of $a_1$. 
\end{example}
\begin{remark}
    The Example \ref{ex:gen_effect} gives the situation involve a multiplication mechanism. The problems involve a similar mechansim can also have generative effects.
\end{remark}

According to Theorem \ref{thm:NLF}, the solutions for the IRM problem are not suitable for problems involving generative effects. Consequently, new methods are needed to model these generative effects, presenting a challenge for future research. Two potential approaches to address the CG problems with generative effects are: 1) transforming the CG problem with generative effects into a CG problem without generative effects by redefining the concepts that generate the distribution, and 2) devising a new methodology to model the generative effects.






\section{Conclusion}
This paper endeavors to establish a comprehensive theory of compositional generalization devoid of specific task assumptions. Task assumptions often confine the analysis to particular problems, impeding a holistic view of the compositional generalization (CG) issue. We commence with Definition \ref{def:cg}, delineating the problem with two fundamental CG properties: 1) each datum embodies clear concepts, and 2) composition rules interconnect data with diverse concepts. Building upon this overarching definition, we establish the no free lunch theorem (Theorem \ref{thm:NLF}), asserting that no general solutions for CG tasks. Leveraging this insight, we derive a generalization bound for any CG tasks (Theorem \ref{thm:information}), along with the conditions for the problem's solvability (Corollary \ref{cor:bound}). Additionally, we categorize CG problems into two types based on the presence of generative effects. We furnish a provable condition for solving problems devoid of generative effects (Theorem \ref{thm:IRM}), leaving those with generative effects as a future challenge. In essence, the importance of this paper lies in furnishing a general theory for CG tasks. When integrated with previous theorems tailored to specific tasks, it contributes to a thorough comprehension of CG.

\section{Limitation}
\label{sec:limitation}

This paper aims to provide a general theory of compositional generalization. As a result, the analysis presented here does not directly offer solutions to specific compositional generalization problems. However, we argue that the theoretical findings in this paper can inspire the design of methods for addressing CG problems. These insights include:
1) Theorem \ref{thm:NLF} suggests that a universal solution for all CG tasks is impossible. Therefore, it is more reasonable to explore which solutions are suitable for specific tasks and develop task-specific solutions accordingly. 2) Theorem \ref{thm:information} further analyzes the conditions for the solvability of CG problems. Future method design should consider these conditions. 3) Section \ref{sec:Generative_Effect} introduces the generative effect to enhance understanding of CG problems and their solutions.

\bibliography{main}
\bibliographystyle{abbrvnat}


\appendix

\section{Other related work}
\label{app_sec:related_work}
\paragraph{Generalization theory with IID. assumption.}
Most generalization theory focus on analyzing the generalization behavior with IID. assumption. One stream of these works rely on the complexity of function space. These methods includes VC dimension  \citep{vapnik2015uniform}, Rademacher Complexity \citep{bartlett2002rademacher} and covering number \citep{shalev2014understanding}. Another kind of works leverage the properties of learning algorithm to analyzing the generalization behavior. These methods include stability of algorithm  \citep{hardt2016train} and information-theoretic methods  \citep{xu2017information,russo2016controlling}. 

\paragraph{Inductive bias of learning algorithm.} The inductive bias of the learning algorithm is focused on the deep learning regime~\citep{fu2023learning}. The inductive bias of the learning algorihtm refers to the preference of the learning algorithm to select certain functions over other functions even though these functions have negliable difference in training data. Because the deep models are so powerful to fit any data and the tradictional way to understand the model's behavior based on the complexity of the function space failed. The researchers turn to understand the learning algorithm from the perspective of learning algorithm~\citep{morwani2021inductive,brutzkus2017sgd,nacson2019stochastic,li2018learning,neyshabur2017implicit}. 

\paragraph{Content and Style generalization} Another line of out of distribution that close to compositional generalization is domain generalization considering the style and content \citep{jing2019neural,jin2022deep}.
This kind can be regard as two compositional factors: style and content. Their work, they usually consider rich content with limited styles, ususally 2. The compositional generalization usually has more complex combination, in the sense that all the factors can have rich factor values.

\section{No Free Lunch Theory}
\begin{theorem} 
Under Assumption \ref{ass:simplify_func_learning_algorithm}, for all valid division of $S$ and $U$, and any $(\mathcal{A}_1,\mathcal{F}_1),(\mathcal{A}_2,\mathcal{F}_2)$, satisfying $|\mathcal{F}_1|=|\mathcal{F}_2|$, we have
    \begin{equation}
    \sum_{T\in \cT} \pr(\Tilde{\bdT}=T|\mathcal{F}_1,\mathcal{A}_1,\bP_{S}^{(T)})=\sum_{T\in \cT} \pr(\Tilde{\bdT}=T|\mathcal{F}_2,\mathcal{A}_2,\bP_{S}^{(T)}),
\end{equation}
    where $\Tilde{\bdT}=\Tilde{T}_{\bdf_S}$ and $\bP_{S}^{(T)}$ is the support distribution generated using the function $T$.
\end{theorem}

\begin{proof}
    We divided the proof into three parts: base case, induction step and conclusion.

    \textbf{1) Base case}
    
    In the base case, we consider the situation where $S=\emptyset$.
    According to Assumption \ref{ass:simplify_func_learning_algorithm} ($i$) and ($ii$), for all composition rule $T$ and any function space $\mathcal{F}$ satisfying the assumption, there is only one function $f\in \mathcal{F}$, such that $\Tilde{T}_f=T$. When $S=\empty$, all learning algorithm cannot perform better than randomly choose a function from the function space. As a result, for any composition rule $T$, any learning algorithm $\mathcal{A}(\cdot)$ and function space $\mathcal{F}$, which satisfying Assumption \ref{ass:simplify_func_learning_algorithm}, we have $\pr(\tdT=T|\mathcal{F},\mathcal{A})=\frac{1}{|\mathcal{F}|}$. Because $\mathcal{F}_1=\mathcal{F}_2$, obviously we have
    $\sum_{T\in \cT} \pr(\tdT=T|\mathcal{F}_1,\mathcal{A}_1)=\sum_{T\in \cT} \pr(\tdT=T|\mathcal{F}_2,\mathcal{A}_2)$. Therefore, we have
    \begin{equation}
    \label{eq:start}
    \begin{aligned}
        &\sum_{T\in \cT} \pr(\tdT=T|\mathcal{F}_1,\mathcal{A}_1,\bP_{S})=\sum_{T\in \cT} \pr(\tdT=T|\mathcal{F}_1,\mathcal{A}_1)\\&=\sum_{T\in \cT} \pr(\tdT=T|\mathcal{F}_2,\mathcal{A}_2)=\sum_{T\in \cT} \pr(\tdT=T|\mathcal{F}_2,\mathcal{A}_2,\bP_{S})
    \end{aligned}
    \end{equation}

    \textbf{2) Induction step}

    We consider two set $S_i\subset S_{i+1} \subset S$ and $S_i \cup \lbrace e_i \rbrace=S_{i+1}$, where $e_i\in E$. Obviously, we have $|S_{i+1}|=|S_i|+1$. We assume that
    \begin{equation}
    \label{eq:NFL_ass}
        \sum_{T\in \cT} \pr(\Tilde{\bdT}=T|\mathcal{F}_1,\mathcal{A}_1,\bP_{S_i}^{(T)})=\sum_{T\in \cT} \pr(\Tilde{\bdT}=T|\mathcal{F}_2,\mathcal{A}_2,\bP_{S_i}^{(T)})
    \end{equation} holds. Then, we have:
    
    \begin{equation} 
    \label{eq:NFL_Ind1}
    \begin{aligned}
    &\pr(\tdT=T|\mathcal{F}_1,\mathcal{A}_1,\bP_{S_{i+1}})\\&=\pr(\tdT=T|\mathcal{F}_1,\mathcal{A}_1,\bP_{S_i},\bP_{e_i})    \\&\overset{(\heartsuit)}{=}\frac{\pr(\bP_{e_i}|\mathcal{F}_1,\mathcal{A}_1,\bP_{S_i},\tdT=T)}{\pr(\bP_{e_i})}\pr(\tdT=T|\mathcal{F}_1,\mathcal{A}_1,\bP_{S_i})
    \\&\overset{(\Diamond)}{=}\frac{\pr(e|S_i)}{\pr(\bP_{e_i})}\pr(\tdT=T|\mathcal{F}_1,\mathcal{A}_1,\bP_{S_i}),
    \end{aligned}
    \end{equation}
    where $(\heartsuit)$ is due to the Bayes rule and $(\Diamond)$ is due to the properties $(ii)$ of Definition \ref{def:cg}.

    Therefore, we have
    \begin{equation}
    \label{eq:iterative}
    \begin{aligned}
        \sum_{T\in \mathcal{T}}\pr(\tdT=T|\mathcal{F}_1,\mathcal{A}_1,\bP_{S_{i+1}})&\overset{(\clubsuit)}{=}\sum_{T\in \mathcal{T}}\frac{\pr(e|S_i)}{\pr(\bP_{e_i})}\pr(\tdT=T|\mathcal{F}_1,\mathcal{A}_1,\bP_{S_i})
        \\&=\frac{\pr(e|S_i)}{\pr(\bP_{e_i})}\sum_{T\in \mathcal{T}}\pr(\tdT=T|\mathcal{F}_1,\mathcal{A}_1,\bP_{S_i})
        \\&\overset{\triangle}{=}\frac{\pr(e|S_i)}{\pr(\bP_{e_i})}\sum_{T\in \mathcal{T}}\pr(\tdT=T|\mathcal{F}_2,\mathcal{A}_2,\bP_{S_i})
        \\ &= \sum_{T\in \mathcal{T}}\pr(\tdT=T|\mathcal{F}_2,\mathcal{A}_2,\bP_{S_{i+1}}),
    \end{aligned}
    \end{equation}
    where $(\clubsuit)$ is due to Equation \ref{eq:NFL_Ind1} and Equation $(\triangle)$ is due to Equation \ref{eq:NFL_ass}.

    \textbf{3) Conclusion}

    Given $S$, we can construct $\lbrace S_0,S_1,S_2,\cdots,S_{|S|} \rbrace$, such that $\forall i, ~ S_i\subset S$ and $ S_i \cup \lbrace e_i \rbrace=S_{i+1}$. We constraint that $S_0=\empty$ and $S_{|S|}=S$. Then, we can iterative apply the Eq. ($\ref{eq:start}$) and Eq. ($\ref{eq:iterative}$) for $ i\in [0,|S|)$ to obtain that 
    \begin{equation}
    \sum_{T\in \cT} \pr(\tdT=T|\mathcal{F}_1,\mathcal{A}_1,\bP_{S})=\sum_{T\in \cT} \pr(\tdT=T|\mathcal{F}_2,\mathcal{A}_2,\bP_{S}).
    \end{equation}
    Therefore, we establish the theorem.
\end{proof}

\section{Generalization Bound}
\label{app_sec:learnablity}
\subsection{Preliminary: Definition and useful lemma}
In the following, we give the measure for the distribution, i.e. Wasserstein Distance  and the some common used function assumption, i.e. Lipschitz assumption and homeomorphis assumption.
\begin{definition}
    (Wasserstein Distance). For any $p\geq 1$, the $p$-Wasserstein distance between two pobability measures $\bP,\bQ$ on the space $\mathcal{W}$ with metric $d_{\mathcal{W}}$ is defined as:
    \begin{equation}
        \bW_p(\bP,\bQ)= \inf_{M\in \Gamma(\bP,\bQ)}(\bE_{(W,W')\sim M}[d^p_{\mathcal{W}}(W,W')])^{1/p},
    \end{equation}
    where $\Gamma(\bP,\bQ)$ denotes the collection
    of all measures on $W \times W$ with the marginals $\bP$ and $\bQ$ on the first and second factors respectively.
\end{definition}

\begin{definition}
    (Lipschitz) Given two metric spaces $(\cM,d_{\cM})$ and $(\cN,d_{\cN})$, where $d_{\cM}$ and $d_{\cN}$ denote the metrics on $\cM$ and $\cN$. A function $h:\cM \to \cN$ is $L$-Lipschitz if for all $m_1,m_2 \in \cM$, we have $d_{\cN}(h(m_1),h(m_2)) \leq Ld_{\cM}(m_1,m_2)$. 
\end{definition}
\paragraph{Lipschitz assumption is commonly used assumption} 
The majority of research relies on the Lipschitz assumption when analyzing generalization behavior. Some studies attempt to alleviate this assumption by substituting it with its weaker counterpart. However, as the primary focus of this paper does not lie in removing the Lipschitz assumption, we defer this task to future work.

\begin{definition}
    (homeomorphism) A continuous function $f$ is called a homeomorphism if it is a bijection function
    and its inverse function $f^{-1}$ is continuous as well.
\end{definition}
\begin{definition}
    (Total Variation) The total variation between two probability distributions $\bP$ and $\bQ$ on $\mathcal{W}$ is
    \begin{equation}
        \tv(\bP,\bQ)\triangleq \sup_{A\in \mathcal{W}} \lbrace \bP(A)-\bQ(A)\rbrace
    \end{equation}
\end{definition}
\begin{definition}
    \label{def:discrete_metric}
    (Discrete Metric) The discrete metric is $d(x,y)\triangleq \boldsymbol{1}[x\neq y]$, where $\boldsymbol{1}$ is the indicator function.
\end{definition}
\begin{lemma}
    \label{lem:rad}
    (Rademacher Complexity (from \citet{mohri2018foundations})) Let $\cF$ be a family of functions. Given a distribution $\bP$ and a samples $D_n=\lbrace z_1, \cdots, z_n \rbrace \sim \bP^{\otimes n}$, the following holds for all $g\in \cF$:
    \begin{equation}
        \bE_{D_n \sim \bP^{\otimes n}} [err(\bP,f)-err(D_n,f)] \leq 2\mathcal{R}_n(\cF),
    \end{equation}
    where $\mathcal{R}_n=\bE_{\sigma,D_n}[\sup_{f\in \cF} \frac{1}{n}\sum_{i=1}^n \sigma_i f(x_i)]$ with $\sigma_i$ being independent uniform random variables taking
    values in $\lbrace -1, +1 \rbrace$.
\end{lemma}

\begin{lemma}
    \label{lem:dual_form_wass}
     For two pobability measures $\bP,\bQ$ on the space $\mathcal{W}$ with metric $d_{\mathcal{W}}$, the $1$-Wasserstein distance between $\bP$ and $\bQ$ can be represented as:
     \begin{equation}
         \bW_1(\bP,\bQ)=\frac{1}{L} \sup_{h \in \cH}\bE_{w\sim \bP} h(w) -\bE_{w\sim \bQ} h(w),
     \end{equation}
     where $\cH$ denotes the function spaces containing function with Lipschitz constant less or equal to $L$.
\end{lemma}


\subsection{Proof of Theorem \ref{thm:information}}
\begin{lemma}
    
    \label{lem:gen_bound}
    If $err$ is $L$-Lipschitz with respect to its second argument,under Assumption \ref{ass:GB_convergence}, we have
    \begin{equation}
        \lv \bE_{D_n,f\sim \mathcal{A}(D_n)} [err(\bP_U,f)-err(D_n,f)] \rv  \leq GenIID+ L\bW_1 (\bP_{\bdf_S},\bP_{\bdf_E})+\mathcal{O}\left( \epsilon \right),  
    \end{equation}
    where $\mathcal{R}_n=\bE_{\sigma,D_n}[\sup_{f\in \cF} \frac{1}{n}\sum_{i=1}^n \sigma_i f(x_i)]$ with $\sigma_i$ being independent uniform random variables taking values in $\lbrace -1, +1 \rbrace$.
\end{lemma}


\begin{proof}

We can decomposite $err(\bP_{U},f)-err(D_n,f)$ as 
\begin{equation}
\begin{aligned}
    & \bE_{D_n,f\sim \mathcal{A}(D_n)} [err(\bP_U,f)-err(D_n,f)] 
    \\& = \underbrace{ err(\bP_{S},\bdf_{D_n})-err(D_n,\bdf_{D_n}) }_{(1)} + \underbrace{err(\bP_{U},\bdf_{D_n})-err(\bP_{S},\bdf_{D_n}) }_{(2)}.
\end{aligned}
\end{equation}
Because $(1)$ is the generalization bound in IID situation, we can upperbound it with any IID bound. Therfore, we can bound "$(1)$" term with $GenIID$ to denotes any upper bound of IID. Then, we only need to focus on the $(2)$ term, which is the essential part of CG.

According to Assumption \ref{ass:GB_convergence}, we can ensure for any $S_1,S_2$ satisfying that for all $S_1 \subset S_2$, we have
\begin{equation}
    \bE_{f\sim \bP_{\bdf_{S_2}}} err(\bP_{S_1},f)=err(\bP_{S_1},\bdf_{S_2})  \leq \max_{e \in S_1} err(\bP_e,\bdf_{S_2}) \leq \max_{e \in S_2} err(\bP_e,\bdf_{S_2})\overset{(\star)}{=} \mathcal{O}(\epsilon),
\end{equation}
where $(\star)$ is due to Assumption \ref{ass:GB_convergence}.

Therefore, we have
\begin{equation}
\begin{aligned}
    & err(\bP_{U},\bdf_S)- err(\bP_{S},\bdf_S)
    \\ =& err(\bP_{U},\bdf_S)-err(\bP_{U},\bdf_E)
        +(err(\bP_{U},\bdf_E)-err(\bP_{S},\bdf_S)) 
    \\ \leq &  err(\bP_{U},\bdf_S)- err(\bP_{U},\bdf_E)+\mathcal{O}(\epsilon)
\end{aligned}.
\end{equation}
    According to Lemma \ref{lem:dual_form_wass}, we have 
    \begin{equation}
        \label{eq:dual_form_wass}
        \bW_1(\bP,\bQ)=\frac{1}{L} \sup_{h \in \cH}\bE_{w\sim \bP} h(w) -\bE_{w\sim \bQ} h(w).
    \end{equation}
    By replacing $h$ in Equation \ref{eq:dual_form_wass} with $err(\bP_{U},\cdot)$, $\bP$ in Equation \ref{eq:dual_form_wass} with $\bP_{\bdf_S}$ and $\bQ$ in Equation \ref{eq:dual_form_wass} with $\bP_{\cF_c}$ we obtain that
    \begin{equation}
    \begin{aligned}
        \bW_1(\bP_{\bdf_S},\bP_{\bdf_E}) & \leq \frac{1}{L} \left( \bE_{f\sim \bP_{\bdf_S}} err(\bP_{U},f)-\bE_{f\sim \bP_{\bdf_E}} err(\bP_{U},f) \right)
        \\&= \frac{1}{L} \left(  err(\bP_{U},\bdf_S)- err(\bP_{U},\bdf_E) \right) .
    \end{aligned}
    \end{equation}
    By rearranging the equation, we obtain that 
    \begin{equation}
        err(\bP_{U},\bdf_S)- err(\bP_{U},\bdf_E) \leq L \bW_1(\bP_{\bdf_S},\bP_{\bdf_E}).
    \end{equation}
    Combining the equations above, the result is established.
\end{proof}



\begin{theorem}
Given training data $D_n\in \mathcal{Z}^n$ sampled from the support distribution $\bP_S$, learning algorithm $\mathcal{A}$ and function space $\mathcal{F}$, if $err$ is $L$-bounded, then we have
\begin{equation*}
    \lv \bE_{D_n,f\sim \mathcal{A}(D_n)} [err(\bP_U,f)-err(D_n,f)] \rv  \leq GenIID+ L\Phi(I(\bdf_S;\bdT|\bP_{S}^{(\bdT)}))+\mathcal{O}\left( \epsilon \right),
\end{equation*}
where $\epsilon=\max \limits_{(a,b)\in E} \min \limits_{f \in \mathcal{F}} err(\bP_{a,b},f)$, $I(\cdot;\cdot)$ denotes the mutual information, $GenIID$ denote any generalization error bound with IID assumption and $\Phi(x)\triangleq \sqrt{\min\lbrace x/2,1-\exp(-x) \rbrace}$.
\end{theorem}

\begin{proof}
    Start from Lemma \ref{lem:gen_bound}, we set the metric between the function space, i.e. $d_{\cF}$ , as the discrete metric as defined in Definition \ref{def:discrete_metric}. Based on this metric, because the $err(\cdot)$ is $L$-bounded, we have for any distribution $\bQ$ and $f_1,f_2 \in \cF$,  $\frac{\lv err(\bQ,f_1)-err(\bQ,f_2)\rv}{d_{\cF}(f_1,f_2)}\leq \frac{\lv L-0 \rv}{1}=L$, i.e. the $err(\cdot)$ is $L$-Lipschitz.

    Then, we can bound $\bW_1 (\bP_{\bdf_S},\bP_{\bdf_E})$ in Lemma \ref{lem:gen_bound} with $\Phi(I(\bdf_S;\bdT|\bP_{S}^{(\bdT)}))$:
    \begin{equation}
    \label{eq:wasserstein-KL}
        \bW_1 (\bP_{\bdf_S},\bP_{\bdf_E})=\bW_1 (\bP_{\bdf_E},\bP_{\bdf_S})\overset{(\clubsuit)}{=}\tv(\bP_{\bdf_E},\bP_{\bdf_S})\overset{(\heartsuit)}{\leq} \Phi(KL(\bP_{\bdf_E},\bP_{\bdf_S})),
    \end{equation}
    where $(\clubsuit)$ is due the Theorem 6.15 of \citet{villani2009optimal},$(\heartsuit)$ is due to the statement in Theorem 6.5 of \citet{polyanskiy2014lecture} and Lemma 2 of \citet{rodriguez2021tighter}. With some misuses, we denote $\bP_{\bdf_S}$ as $\bP_{\bdf}|\bP_{S}$, where $|$ denotes the condition and the same for $\bP_{\bdf_E}$ and $\bP_{\bdf_U}$. Then, we have
    
    \begin{equation}
    \label{eq:kl-mu}
        \begin{aligned}
            & \kl(\bP_{\bdf_E},\bP_{\bdf_S})=\kl(\bP_{\bdf}|\bP_{E},\bP_{\bdf}|\bP_{S})=\kl(\bP_{\bdf}|(\bP_{S},\bP_{U}),\bP_{\bdf}|\bP_{S})
            \\&=\kl(\bP_{\bdf}|\bP_{U},\bP_{\bdf} )|\bP_{S}=I(\bP_{\bdf};\bP_{U}|\bP_{S})=I(\bP_{\bdf};T|\bP_{S})=I(\bP_{\bdf_S};T|\bP_{S})
        \end{aligned}
    \end{equation}
    Then, taking Equation \ref{eq:wasserstein-KL} and Equation \ref{eq:kl-mu} back into Lemma \ref{lem:gen_bound}, the Theorem is estabilished.
\end{proof}

\subsection{Proof of Corollary \ref{cor:bound}}
\begin{corollary}
    Given the support distribution $\bP_S$, if $err$ is $L$-bounded, and $\lim_{n \to \infty} GenIID =0$,
    the CG problem is solvable if $err(\bP_S,\bdf_S)+\Phi(I(\bdf_S;\bdT|\bP_{S}^{(\bdT)}))=\mathcal{O}(\epsilon)$.
\end{corollary}
\begin{proof}
    Recall that the generalization bound in Theorem \ref{thm:information}, that
    \begin{equation}
    \label{eq:corr_tem_1}
        \lv \bE_{D_n,f\sim \mathcal{A}(D_n)} [err(\bP_U,f)-err(D_n,f)] \rv  \leq GenIID+ \kappa_n L\Phi(I(\bdf_S;\bdT|\bP_{S}^{(\bdT)}))+\mathcal{O}\left( \epsilon \right).
    \end{equation}
    According to the Definition \ref{def:solvability}, the CG is solvable if 
    
    \begin{equation}
    \begin{aligned}
        &\lim_{n\to \infty} \lv \bE_{D_n,f\sim \mathcal{A}(D_n)} [err(\bP_U,f)-err(D_n,f)] \rv \\&= 
         \lv  [err(\bP_U,\bdf_S)-err(\bP_S,\bdf_S)] \rv
        \\&\overset{(\square)}{=} \lim_{n \to \infty} (GenIID+ \kappa_n L\Phi(I(\bdf_S;\bdT|\bP_{S}^{(\bdT)}))+\mathcal{O}\left( \epsilon \right))
        \\&\overset{(\star)}{=} \lim_{n \to \infty}  L\Phi(I(\bdf_S;\bdT|\bP_{S}^{(\bdT)}))+\mathcal{O}\left( \epsilon \right)
        \\&=L\Phi(I(\bdf_S;\bdT|\bP_{S}^{(\bdT)})+\mathcal{O}\left( \epsilon \right),
    \end{aligned}
    \end{equation}
    where $(\square)$ is due to Equation \ref{eq:corr_tem_1}, and $(\star)$ is due to the condition $\lim_{n \to \infty} GenIID = 0$ and $\lim_{n \to \infty} \kappa_n=1$.

    Then, we have
    \begin{equation}
        err(\bP_U,\bdf_S) \leq err(\bP_S,\bdf_S)+L\Phi(I(\bdf_S;\bdT|\bP_{S}^{(\bdT)})+\mathcal{O}\left( \epsilon \right).
    \end{equation}
    Under the condition that $err(\bP_S,\bdf_S)+\Phi(I(\bdf_S;\bdT|\bP_{S}^{(\bdT)}))=\mathcal{O}(\epsilon)$,
    we have 
\begin{equation}
    err(\bP_U,\bdf_S)=\mathcal{O}(\epsilon).
    \end{equation}
    According to the definition \ref{def:solvability}, the problem is solvable. Therefore, the Corollary holds.
\end{proof}

\subsection{Explore the tightness of the generalization bounds}
\label{app_sec:tightness}

\begin{figure}[h]
    \centering
    \includegraphics[width=0.8\linewidth]{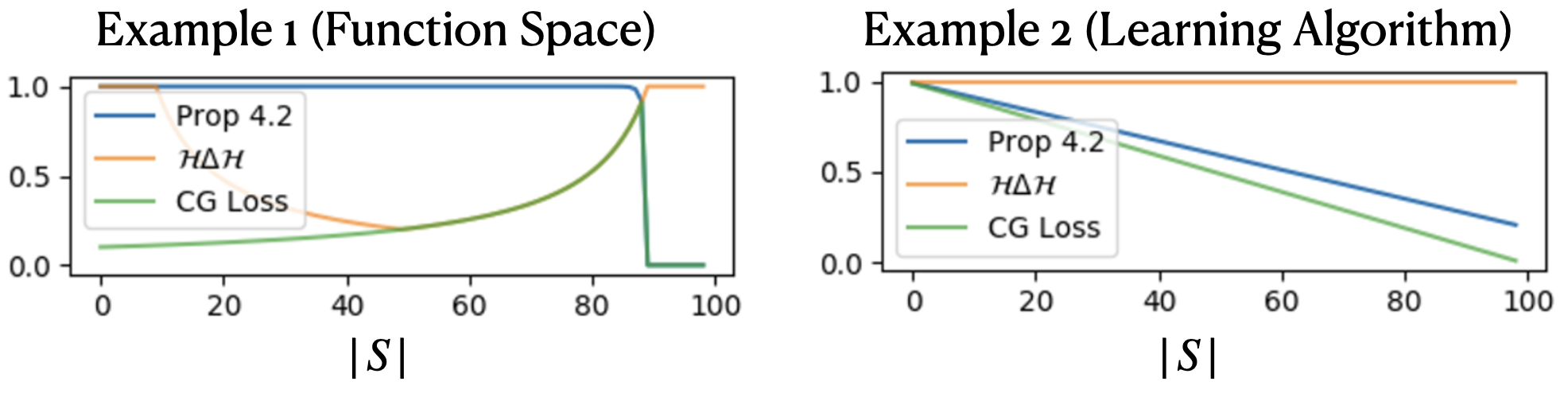}
    \caption{\textbf{Generalization bounds on the toy problem.} The example 1 considers the case where the function space has some bias while the learning algorithm has no bias. The example 2 consider the learning algorithm has certain bias while the function space is powerful to fit data. We find that 1) our bounds can capture the decrease of generalization error in example 1 and 2) our can align with the generalization error in example 2.}
    \label{fig:bounds_compare}
\end{figure}

Our ability to assert whether our bound is tighter or looser than previous bounds is contingent upon considering the nuanced intricacies of the problems at hand. According to whether the problem is more influenced by the design of learning algorithm or the function space. We have delineated the issue into two distinct categories 
\begin{itemize}
    \item \textbf{Function space dominated problem.} In this problem, we posit that the learning algorithm randomly selects functions that minimize loss on the training data, within a function space tailored specifically for the problem at hand. 
    \item \textbf{learning algorithm dominated problem.} Here, we assume that the function space is pwerful enough to accommodate any distribution, while the learning algorithm is inclined to favor certain functions over others, provided they minimize loss on the training data.
\end{itemize}

To illustrate these concepts, we provide two examples: Example 1 exemplifies a scenario where the function space dominates, whereas Example 2 exemplifies a scenario where the learning algorithm holds sway.
In summation, our analysis yields the following conclusion: \textit{We ascertain that our bound achieves greater tightness in the context of the learning algorithm dominated problem. Conversely, in the scenario where the function space dominates, our bound achieves greater tightness solely when provided with an extensive array of supporting distributions.}

\paragraph{Example setting} We consider thet $|A|=|B|=10$. The two examples are explored. \textbf{Example 1} The function space $\cF$ has the properties that 1) For each function $f\in \cF$, we have $\sum_{a,b} \mathbb{I}_{err(\bP_{a,b},f)}=10$ or $\sum_{a,b} \mathbb{I}_{err(\bP_{a,b},f)}=0$. 2) $\forall a \in A,b \in B, err(\bP_{a,b},f) =1$ or $ err(\bP_{a,b},f) =0$. The learning algorithm satisfies that $\forall (a,b) \in S$ and for all $f \in \supp \bP_{\bdf_S}$ and for all $(a,b) \in S$, we have $ err(\bP_{a,b},f)=0 $. \textbf{Example 2} For all $f \in \supp \bP_{\bdf_S}$, for all $(a,b) \in S$, we have $ err(\bP_{a,b},f)=0$ and for all $(a,b) \notin S$, we have $ err(\bP_{a,b},f)=0$ with probability $c_{a,b}\frac{|S|}{|E|}$ else $ err(\bP_{a,b},f)=1$, where $c_{a,b}$ is a ramdonly assigned value for each $(a,b)$ and it takes value between 0.8 and 1. We choose the distance measure $d_{\cF}(f_1,f_2)=\sup_{(a,b)\in E}\bE(\lv err(\bP_{a,b},f_1)-err(\bP_{a,b},f_2)\rv)$.

\begin{remark}
    In \textbf{Example 1}, we delve into the bias stemming from the function space. Here, the function space is relatively constrained, containing only a limited set of functions, including the correct one that attains zero loss. The learning algorithm uniformly selects a function only if it achieves minimal loss on the support distribution. Consequently, the learning algorithm exhibits no inherent bias towards specific functions as long as they achieve minimal loss on the support distributions. In \textbf{Example 2}, we explore the bias inherent in the learning algorithm. In this instance, the function space is expansive, encompassing all possible outputs. However, the learning algorithm may assign varying probabilities to functions that achieve zero loss on the support distribution.
\end{remark}

We revisit the results of \citet{ben2010theory}, which is the work that we aim to compare with. \citet{ben2010theory} has the following conclusion that:
\begin{equation}
    err(\bP_{U},f)-err(\bP_{S},f)\leq d_{\cF\Delta\cF}(\bP_{U},\bP_{S})+\mathcal{O}(\epsilon),
\end{equation}
where the $d_{\cF\Delta\cF}$ is defined as
 $d_{\cF\Delta\cF}  \triangleq 2 \sup_{f,f'\in \cF} \lv \bE_{z\sim \bP_{U}} [f(x) \neq f'(x)] -\bE_{z\sim \bP_{S}} [f(x) \neq f'(x)] \rv$.

\paragraph{Results} We compute the generalization bounds and error depicted in Fig.~\ref{fig:bounds_compare}, revealing two key observations: Our bound effectively incorporates the impact of the support distribution. In \textbf{Example 1}, our generalization bound accurately reflects the decreasing trend of the generalization error. Similarly, in \textbf{Example 2}, our bound aligns with the generalization trends across various support distributions. Our bound accounts for the influence of the learning algorithm. Notably, in \textbf{Example 2}, the approach proposed by \citet{ben2010theory} fails to capture the dynamics accurately. This failure can be attributed to its predominant focus on the function space influence, whereas our analysis recognizes the dominance of the learning algorithm's influence in this example.

\subsection{Discussion of Prior Works}
\begin{table}[h]
    \label{tab:additional_contraint_related_work}
    \centering
    \begin{tabular}{c|c|c}
    \hline
         & $\Phi^{(x)}$  & $\Phi^{(y)}$\\
         \hline
        \citet{dong2022first} & $\Phi_1^{(x)}(a,\zeta)=(\Tilde{\Phi}_1^{(x)}(a,\zeta),\boldsymbol{0}) \quad \Tilde{\Phi}_1^{(x)}(b,\zeta)=(\boldsymbol{0},\Tilde{\Phi}_1^{(x)}(b,\zeta))$ & No \\
        \citet{wiedemer2024compositional} & $\Phi^{(x)}_1 (a,\zeta)+\Phi^{(x)}_1 (a,\zeta)=(a,b,\zeta)$ & No \\
        \hline
    \end{tabular}
    \caption{The additional contraint imposed by the prior works upon the Equation \ref{eq:IRM_gen}. The term ``No" in the Table indicates that there are no additions constraint. The $\boldsymbol{0}$ denotes the vector whose every entry is zero.}
    \label{tab:my_label}
\end{table}

\begin{definition}
    \label{def:gen_func}
    (Generation Function) 
    The function $\Phi(\cdot)$ is the generation function for a CG distribution defined in Definition \ref{def:cg}, if for any $(a,b)\in E$ such that the random variable $\bdz_{a,b}\sim \bP_{a,b}$ can be generated according to the generation function $\bdz_{a,b} =\Phi (a,b,\bdzeta)$, where the random variable $\bdzeta \sim \bP_{\bdzeta}$ to denote any other concepts and the noise.
\end{definition}

In the following, we discuss the relation between the generation function and the IRM:

\begin{proposition}
    For any CG distribution defined according Definition \ref{def:cg} whose generation function $\Phi(\cdot)$ has the following form:
    \begin{equation}
\begin{aligned}
    \label{eq:IRM_gen}
    x&=\Phi_1^{(x)}(a,\zeta)+\Phi_2^{(x)}(b,\zeta) \\
    y&=\Phi_1^{(y)}(a,\zeta)+\Phi_2^{(y)}(b,\zeta),
\end{aligned}
\end{equation}
where $z=(x,y)$. Then the distribution has IRM.
\end{proposition}

\begin{remark}
    In Table \ref{tab:additional_contraint_related_work}, we list the additional constraint for the other related work in additional to the Equation \ref{eq:IRM_gen}.
\end{remark}

\begin{proof}
    According to the Defintion \ref{def:gen_func}, we have $z=\Phi(a,b,\zeta)$. Then, according to the Equation \ref{eq:IRM_gen}, we can decomposite the generation function into the following form:
    \begin{equation}
        z=(x,y)=(\Phi_1^{(x)}(a,\zeta)+\Phi_2^{(x)}(b,\zeta),\Phi_1^{(y)}(a,\zeta)+\Phi_2^{(y)}(b,\zeta))=\Phi(a,b,\zeta).
    \end{equation}
With this generation process, the corresponding composition rule can be calculated as:
\begin{equation}
\label{eq:decom_CR_IRM}
\begin{aligned}
    T_{a_1 \to a_2}(z)&=z+(\Phi_1^{(x)}(a_2,\zeta)-\Phi_1^{(x)}(a_1,\zeta),\boldsymbol{0})\\
    T_{b_1 \to b_2}(z)&=z+(\boldsymbol{0},\Phi_1^{(y)}(b_2,\zeta)-\Phi_1^{(y)}(b_1,\zeta)).
\end{aligned}
\end{equation}
It is obviously that the Equation \ref{eq:decom_CR_IRM} holds for any $(a_1,b_1), (a_2,b_2) \in E$. Therefore, we have 
\begin{equation}
\begin{aligned}
    T_{a_1 \to a_2, b_1 \to b_2}&=T_{a_1 \to a_2}\circ T_{b_1 \to b_2}=T_{b_1 \to b_2}\circ T_{a_1 \to a_2}
    \\&= z+(\Phi_1^{(x)}(a_2,\zeta)-\Phi_1^{(x)}(a_1,\zeta),\Phi_1^{(y)}(b_2,\zeta)-\Phi_1^{(y)}(b_1,\zeta))
\end{aligned}
\end{equation}
holds for any $(a_1,b_1), (a_2,b_2) \in E$.
Accroding to Definition \ref{def:IRM}, we can obtain that the distribution has IRM.

\end{proof}

\end{document}